%% file: approximate-sparse-coding-tech-report.tex
\documentclass{report}

\usepackage{afterpage}

\usepackage{float}

\usepackage{longtable}

\usepackage{graphicx}

\usepackage{pdflscape}

\usepackage[numbers,sort&compress]{natbib}

\usepackage{psfrag}

\usepackage[unicode=true,
  linktocpage,
  linkbordercolor={0.5 0.5 1},
  citebordercolor={0.5 1 0.5},
  linkcolor=blue]{hyperref}

\title{Recklessly Approximate\\Sparse Coding}
\author{Misha Denil}

\floatstyle{ruled}
\newfloat{Program}{htbp}{lop}[chapter]

\usepackage{caption}
\usepackage{subcaption}
\usepackage{amsmath}
\usepackage{amssymb}
\usepackage{amsfonts}
\usepackage{amsthm}
\usepackage{color}
\usepackage[usenames,dvipsnames,svgnames,table]{xcolor}

\newtheorem{prop}{Proposition}

\renewcommand{\a}{\alpha}
\newcommand{\prox}{\mathrm{prox}}
\newcommand{\R}{\mathbb{R}}
\newcommand{\grad}{\nabla}

\newcommand{\T}{\mathrm{T}}
\newcommand{\soft}{\operatorname{soft}}

\begin{document}

\maketitle
\input{abstract}

\input{introduction/introduction}
\input{proximal-optimization/proximal-optimization}
\input{sparse-coding/sparse-coding}
\input{approximation/approximation}

\input{experiments/experiments}

\input{conclusion/conclusion}

\nocite{kriz09}
\bibliographystyle{plain}
\bibliography{master-report}{}

\end{document}

%% file: abstract.tex
\begin{abstract}                

  It has recently been observed that certain extremely simple feature
  encoding techniques are able to achieve state of the art performance
  on several standard image classification benchmarks including deep
  belief networks, convolutional nets, factored RBMs, mcRBMs,
  convolutional RBMs, sparse autoencoders and several others.
  Moreover, these ``triangle'' or ``soft threshold'' encodings are
  extremely efficient to compute.  Several intuitive arguments have
  been put forward to explain this remarkable performance, yet no
  mathematical justification has been offered.

  The main result of this report is to show that these features are
  realized as an approximate solution to the a non-negative sparse
  coding problem.  Using this connection we describe several variants
  of the soft threshold features and demonstrate their effectiveness
  on two image classification benchmark tasks.




\end{abstract}


%% file: introduction/introduction.tex
\chapter{Introduction}

\section{Setting}

Image classification is one of several central problems in computer
vision.  This problem is concerned with sorting images into categories
based on the objects or types of objects that appear in them.  An
important assumption that we make in this setting is that the object
of interest appears prominently in each image we consider, possibly in
the presence of some background ``clutter'' which should be ignored.
The related problem of object localization, where we predict the
location and extent of an object of interest in a larger image, is not
considered in this report.

Neural networks are a common tool for this problem and have have been applied in
this area since at least the late 80's~\cite{lecun-89d}.  More recently the
introduction of contrastive divergence~\cite{hintonosinderoteh06} has lead to an
explosion of work on neural networks to this task.  Neural network models serve
two purposes in this setting:
\begin{enumerate}
\item They provide a method to design a dictionary of primitives to
  use for representing images.  With neural networks the dictionary
  can be designed through learning, and thus tailored to a specific
  data set.
\item They provide a method to encode images using this dictionary to
  obtain a feature based representation of the image.
\end{enumerate}
Representations constructed in this way can be classified using a standard
classifier such as a support vector machine.  Properly designed feature based
representations can be classified much more accurately than using the raw pixels
directly.  Neural networks have proved to be effective tools for constructing
these representations~\cite{ranzanto2011, hinton2010, mnih2010}.

A major barrier to applying these models to large images is the number
of parameters required.  Designing features for an $n\times n$ image
using these techniques requires learning $O(n^4)$ parameters which
rapidly becomes intractable, even for small images.  A common solution
to this problem is to construct features for representing image patches
rather than full images, and then construct feature representations of
full images by combining representations of their patches.

While much effort has been devoted to designing elaborate feature
learning methods~\cite{ranzato2010a, ranzato2010b, quocle2012,
  goodfellow2012, courville2011, Rifai+al-2011,
  Dauphin-et-al-NIPS2011}, it has been shown recently that provided
the dictionary is reasonable then the encoding method has a far
greater effect on classification performance than the specific choice
of dictionary~\cite{coates2011encoding}.

In particular,~\cite{coates2010analysis} and~\cite{coates2011encoding}
demonstrate two very simple feature encoding methods that outperform a
variety of much more sophisticated techniques.  In the aforementioned
works, these feature encoding methods are motivated based on their
computational simplicity and effectiveness.  The main contribution of
this report is to provide a connection between these features and
sparse coding; in doing so we situate the work
of~\cite{coates2010analysis} and~\cite{coates2011encoding} in a
broader theoretical framework and offer some explanation for the
success of their techniques.

\section{Background}

In~\cite{coates2010analysis} and~\cite{coates2011encoding}, Coates and Ng found
that two very simple feature encodings were able to achieve state of the art
results on several image classification tasks.  In~\cite{coates2010analysis}
they consider encoding image patches, represented as a vector in $x \in \R^N$
using the so called ``K-means'' or ``triangle'' features to obtain a
$K$-dimensional feature encoding $z^{\text{tri}}(x) \in \R^K$, which they
define elementwise using the formula
\begin{align}
  z_k^{\text{tri}}(x) &= \max\{0,\mu(x)-||x-w_k||_2\} \enspace,
  \label{eq:triangle}
\end{align}
where $\{w_k\}_{k=1}^K$ is a dictionary of elements obtained by clustering data
samples with K-means, and $\mu(x)$ is the average of $||x-w_k||_2$ over $k$.  In
~\cite{coates2011encoding} they consider the closely related ``soft threshold''
features, given by
\begin{align}
  \label{eq:st-features}
  z_k^{\text{st}}(x) = \max\{0, w_k^\T x - \lambda \} \enspace,
\end{align}
with $\lambda \ge 0$ as a parameter to be set by cross validation.  These
feature encodings have proved to be surprisingly effective, achieving state of
the art results on popular image classification benchmarks.  However, what makes
these features especially appealing is their simplicity.  Given a dictionary,
producing an encoding requires only a single matrix multiply and threshold
operation.

We note that the triangle and soft threshold features are merely
slight variations on the same idea.  If we modify the triangle
features to be defined in terms of squared distances, that is we
consider
\begin{align*}
  z_k^{\text{tri}}(x) &= \max\{0,\mu_2(x)-||x-w_k||_2^2\} \enspace,
\end{align*}
in place of Equation~\ref{eq:triangle}, with $\mu_2(x)$ taking the
average value of $||x-w_k||_2^2$ over $k$, we can then write
\begin{align*}
  \mu_2(x)-||x-w_k||_2^2 &= 2w_k^\T x - \frac{2}{n}\sum_{i=1}^n w_i^\T
  x - w_k^\T w_k + \frac{1}{n}\sum_{i=1}^n w_i^\T w_i \enspace.
\end{align*}
If we constrain the dictionary elements $w_i$ to have unit norm as
in~\cite{coates2011encoding} then the final two terms cancel and the
triangle features can be rewritten as
\begin{align*}
  z_k^{\text{tri}}(x) &= 2\max\{0, w_k^\T x - \lambda(x)\} \enspace,
\end{align*}
which we can see is just a scaled version of soft threshold features,
where the threshold,
\begin{align*}
  \lambda(x) &= \frac{1}{n}\sum_{i=1}^n w_i^\T x \enspace,
\end{align*}
is chosen as a function of $x$ rather than by cross validation.  In this report
we consider only the soft threshold features, but their similarity with the
triangle features is suggestive.

\section{Related work}

Similar approaches to feature encoding are well known in the computer vision
literature under the name of vector quantization.  In this approach, data are
encoded by hard assignment to the nearest dictionary element. Van Gemert et
al.~\cite{van2008kernel} consider a softer version of this idea and find that
using a kernel function for quantization rather than hard assignment leads to
better performance.  Bourdeau et al.~\cite{Boureau2010} consider a similar soft
quantization scheme but find that sparse coding performs better still.

Following the success of~\cite{coates2010analysis}, the triangle and
soft threshold features (or slight variants thereof) have been applied
in several settings.  Blum et al.~\cite{blum2011} use triangle
features for encoding in their work on applying unsupervised feature
learning to RGB-D data using a dictionary designed by their own
convolutional K-means approach.  Knoll et al.~\cite{KnollF12} apply
triangle features to image compression using PAQ.  An unthresholded
version of triangle features was used in~\cite{savva2011} as the
low-level image features for a system which extracts and redesigns
chart images in documents.

In~\cite{netzer2011} and~\cite{coates2011text} soft threshold features
are used for the detection and recognition of digits and text
(respectively) in natural images.  The same features have also been
employed in~\cite{NIPS2011_1368} as part of the base learning model in
a system for selecting the receptive fields for higher layers in a
deep network.

Jia et al.~\cite{jia12} consider both triangle and soft threshold
features for encoding image patches and investigate the effects of
optimizing the spatial pooling process in order to achieve better
classification accuracy.

The work most similar to that found in this report is the work of Gregor and
LeCun on approximations of sparse coding~\cite{gregor-lecun10}.  Like us, they
are interested in designing feature encoders using approximations of sparse
coding; however, their technique is very different than the one we consider here.

The chief difficulty with sparse coding is that the encoding step
requires solving an $\ell_1$ regularized problem, which does not have
a closed form solution.  For example, if we want to extract features
from each frame of a video in real time then sparse coding is
prohibitively slow.  In~\cite{gregor-lecun10} the authors design
trainable fixed cost encoders to predict the sparse coding features.
The predictor is designed by taking an iterative method for solving
the sparse coding problem and truncating it after a specified number
of steps to give a fixed complexity feed forward predictor.  The
parameters of this predictor are then optimized to predict the true
sparse codes over a training set.

Both our work and that of~\cite{gregor-lecun10} is based on the idea of
approximating sparse coding with a fixed dictionary by truncating an
optimization before convergence, but we can identify some key differences:
\begin{itemize}
\item The method of~\cite{gregor-lecun10} is trained to predict sparse codes on
  a particular data set with a particular dictionary.  Our method requires no
  training and is agnostic to the specific dictionary that is used.
\item We focus on the problem of creating features which lead to good
  classification performance directly, whereas the focus
  of~\cite{gregor-lecun10} is on predicting optimal codes.  Our experiments show
  that, at least for our approach, these two quantities are surprisingly
  uncorrelated.
\item Although there is some evaluation of classification performance of
  truncated iterative solutions without learning in~\cite{gregor-lecun10}, this
  is only done with a coordinate descent based algorithm.  Our experiments
  suggest that these methods are vastly outperformed by truncated proximal
  methods in this setting.
\end{itemize}
Based on the above points, the work in this report can be seen as complimentary
to that of~\cite{gregor-lecun10}.

\section{Structure of this report}

The remainder of this report is structured as follows:
\begin{itemize}
\item In Chapter~\ref{chap:optimization} we give some general background on the
  proximal gradient method in optimization, and introduce some extensions of
  this method that operate in the dual space.
\item In Chapter~\ref{chap:sparse-coding} we introduce sparse coding and outline
  four specific algorithms for solving the encoding problem.
\item Chapter~\ref{chap:approximation} contains the main result of
  this report, which is a connection between the soft threshold
  features and sparse coding through the proximal gradient algorithm.
  Using this connection we outline four possible variants of the soft
  threshold features, based on the sparse encoding algorithms from
  Chapter~\ref{chap:sparse-coding}.
\item In Chapter~\ref{chap:experiments} we report on two experiments designed to
  asses the usefulness of these feature variants.
\item In Chapter~\ref{chap:conclusion} we conclude and offer suggestions for
  future work.
\end{itemize}


%% file: proximal-optimization/proximal-optimization.tex
\chapter{Optimization}
\label{chap:optimization}

This chapter provides the necessary optimization background to
support the tools used in this report.  Focusing on objective
functions with a specific separable structure, we discuss the proximal
gradient algorithm with a fixed step size and an extension of this
method that uses a Barzilai-Borwein step size scheme.  We also
consider a variant of the proximal gradient algorithm which operates
in the dual space, and an alteration of that method to make it
tractable for the sparse coding problem.

Throughout this chapter we eschew generality in favour of developing
tools which are directly relevant to the sparse coding problem.  All
of the methods we discuss have more general variants, which are
applicable to a broader class of problems then we are concerned with.
The citations in this chapter can be used to find expositions of these
methods in more general settings.

The presentation in this chapter assumes familiarity with some common
optimization tools, specifically the reader should be familiar with
Lagrangian methods and dualization, which are used here without
justification.  Extensive discussions of the supporting theory of the
Lagrangian dual can be found in any standard text on convex
optimization such as~\cite{bert95} or~\cite{bertsekas89}.

\section{Setting}
In what follows, we concern ourselves with the function
\begin{align}
  f(x) = g(x) + h(x) \enspace,
  \label{eq:raw}
\end{align}
where $g : \R^n\to\R$ is differentiable with Lipschitz derivatives,
\begin{align*}
  ||\grad g(x) - \grad g(y)||_2 \le L||x - y||_2 \enspace,
\end{align*}
and $h : \R^n\to\R$ is convex.  In particular we are interested in
cases where $h$ is not differentiable.  The primary instance of that
will concern us in this report is
\begin{align*}
  f(x) = \frac{1}{2}||W z - x||_2^2 + \lambda ||z||_1 \enspace,
\end{align*}
where $g$ is given by the quadratic term and $h$ handles the
(non-differentiable) $\ell_1$ norm.

We search for solutions to
\begin{align}
  \min_{x} f(x)
  \label{eq:unconstrained}
\end{align}
and we refer to Equation~\ref{eq:unconstrained} as the \emph{unconstrained
  problem}.  It will also be useful to us to consider an alternative
formulation,
\begin{align}
  \min_{x=z} g(x) + h(z) \enspace,
  \label{eq:constrained}
\end{align}
which has the same solutions as Equation~\ref{eq:unconstrained}.  Introducing
the dummy variable $z$ gives us access to the dual space which will be useful
later on.  We refer to this variant as the \emph{constrained problem}.

An object of central utility in this chapter is the $\prox_{h,\rho}$ operator,
which is defined, for a convex function $h$ and a scalar $\rho > 0$ as
\begin{align*}
  \prox_{h,\rho}(x) = \arg\min_u\{h(u) + \frac{\rho}{2}||u - x||_2^2\} \enspace.
\end{align*}
We are often interested in the case where $\rho=1$, and write $\prox_h$ in these
cases to ease the notation.

\section{Proximal gradient}
Proximal gradient is an algorithm for solving problems with the form of
Equation~\ref{eq:unconstrained} by iterating
\begin{align}
  \label{eq:proxgrad}
  x^{t+1} &= \prox_{\a h}(x^t - \a\grad g(x^t)) \\
  \nonumber &= \arg\min_u\{\a h(u) + \frac{1}{2}||u - (x^t - \a\grad
  g(x^t))||_2^2\}
\end{align}
for an appropriately chosen step size $\alpha$.  It can be shown that
if $\a < 1/L$ then this iteration converges to an optimal point of
$f$~\cite{Vandenberghe2011}.

Alternative step size selection methods such as line search and
iterate averaging~\cite{beckteboulle09} are also possible.  One such
scheme is the Barzilai-Borwein scheme~\cite{barzborwein88} which picks
the step size $\a^t$ so that $\a^tI$ approximates the inverse Hessian
of $g$.  This choice of step size can be motivated by the form of the
proximal gradient updates.  We consider approximating the solution to
Equation~\ref{eq:unconstrained} using a quadratic model of $g$ with
diagonal covariance $\a^{-1} I$.  Approximating $g$ in this way, with
a Taylor expansion about an arbitrary point $x^0$, leads to the
problem
\begin{align*}
  \arg\min_x \{ h(x) + g(x^0) + \grad g(x^0)^\T(x-x^0) +
  \frac{1}{2\alpha}||x-x^0||_2^2  \} \enspace.
\end{align*}
Some algebra reveals that we can rewrite this problem as follows:
\begin{align*}
  \arg&\min_x \{h(x) + \grad g(x^0)^\T(x-x^0) +
  \frac{1}{2\alpha}||x-x^0||_2^2 \}
  \\
  &= \arg\min_x \{ \a h(x) + \a \grad g(x^0)^\T(x-x^0) +
  \frac{1}{2}||x-x^0||_2^2 \}
  \\
  &= \arg\min_x \{ \a h(x) + \frac{1}{2}x^\T x - x^\T(x^0 - \a \grad
  g(x^0)) + \frac{1}{2} ||x^0 - \a\grad g(x^0)||_2^2 \}
  \\
  &= \arg\min_x \{ \a h(x) + \frac{1}{2} ||x - (x^0 - \a \grad
  g(x^0))||_2^2 \}
  \\
  &= \prox_{\a h}(x^0 - \a\grad g(x^0)) \enspace,
\end{align*}
showing that the proximal gradient update in Equation~\ref{eq:proxgrad} amounts
to minimizing $h$ plus a quadratic approximation of $g$ at each step.  In the
above calculation we have made repeated use of the fact that adding or
multiplying by a constant does not affect the location of the argmax.  The
Barzilai-Borwein scheme adjusts the step size $\alpha^t$ to ensure that the
model of $g$ we minimize is as accurate as possible.  The general step size
selection rule can be found in~\cite{wright09, barzborwein88}.  We will consider
a specific instance of this scheme, specialized to the sparse coding problem, in
Section~\ref{sec:sparsa}.

\section{Dual ascent}
We now consider the constrained problem.  The constraints in
Equation~\ref{eq:constrained} allow us to form the Lagrangian,
\begin{align*}
  L(x,z,y) = g(x) + h(z) + y^\T(x - z) \enspace,
\end{align*}
which gives us access to the dual problem,
\begin{align*}
  \max_y q(y) = \max_y \min_{x,z} L(x,z,y) \enspace.
\end{align*}
It can be shown that if $y^*$ is a solution to the dual problem then $(x^*, z^*)
= \arg\min_{x,z} L(x,z,y^*)$ is a solution to the primal
problem~\cite{bertsekas89}.  Assuming that $q(y)$ is differentiable, this
connection suggests we compute a solution to the primal problem by forming the
sequence
\begin{align*}
  y^{t+1} = y^t + \alpha\grad q(y^t)
\end{align*}
and estimate the values of the primal variables as
\begin{align}
  (x^{t+1}, z^{t+1}) = \arg\min_{x,z} L(x,z,y^t) \enspace.
  \label{eq:dual-ascent-xz}
\end{align}
Forming this estimate at each step requires no extra computation since computing
the update requires, $\grad q(y) = x^{t+1} - z^{t+1}$.  In in our case the
minimization in Equation~\ref{eq:dual-ascent-xz} splits into separate
minimizations in $x$ and $z$,
\begin{align}
  x^{t+1} &= \arg\min_x\{g(x) + (y^t)^\T x\} \enspace, \label{eq:xt-dualascent} \\
  z^{t+1} &= \arg\min_z\{h(z) - (y^t)^\T z\} \enspace. \label{eq:zt-dualascent}
\end{align}
This method is known as dual ascent~\cite{boyd2011}, and can be shown
to converge under certain conditions.  Unfortunately in the sparse
coding problem these conditions are not satisfied.  As we see in
Chapter~\ref{chap:sparse-coding}, we are often interested in problems
where $g$ is minimized on an entire subspace of $\R^n$.  This is
problematic because if the projection of $y$ into this subspace is
non-zero then the minimization in Equation~\ref{eq:xt-dualascent} is
unbounded and $\grad q(y)$ is not well defined.

\section{Method of multipliers}
A tool to help us work around the shortcomings of dual ascent is the Augmented
Lagrangian, which is a family of functions parametrized by $\rho\ge0$,
\begin{align*}
  L_\rho(x,z,y) = g(x) + h(z) + y^\T(x - z) + \frac{\rho}{2}||x-z||_2^2
  \enspace.
\end{align*}
The function $L_\rho$ is the Lagrangian of the problem
\begin{align}
  \min_{x=z} g(x) + h(z) + \frac{\rho}{2}||x-z||_2^2 \enspace,
  \label{eq:augmented}
\end{align}
which we see has the same solutions Equation~\ref{eq:constrained}.  The
quadratic term in the Augmented Lagrangian gives the dual problem nice
behaviour.  The augmented dual is given by
\begin{align*}
  \max_y q_\rho(y) = \max_y \min_{x,z} L_\rho(x,z,y) \enspace.
\end{align*}
We again consider gradient ascent of the objective $q_\rho$ by forming the
sequence
\begin{align}
  y^{t+1} = y^t + \rho \grad q_\rho(y^t) \enspace,
  \label{eq:momy}
\end{align}
where $\grad q_\rho(y^t) = x^{t+1} - z^{t+1}$ with
\begin{align}
  (x^{t+1}, z^{t+1}) &= \arg\min_{x,z} L_\rho(x,z,y^t) \enspace.
  \label{eq:momxz}
\end{align}
This algorithm is known as the method of multipliers.  The quadratic
term in $L_\rho$ ensures that $\grad q_\rho(y)$ always exists and the
algorithm is well defined.  The use of $\rho$ as the step size is
motivated by the fact that it guarantees that the iterates will be
dual-feasible at each step~\cite{boyd2011}.

It can be shown that Equation~\ref{eq:momy} can be written~\cite{bert95}
\begin{align*}
  y^{t+1} &= \arg\min_\lambda \{ -q(\lambda) + \frac{1}{2\rho}||\lambda -
  y^t||_2^2 \} \\
  &= \prox_{-q,1/\rho}(y^t) \enspace,
\end{align*}
where $q$ is the Lagrangian of Equation~\ref{eq:constrained}.
Comparing this to Equation~\ref{eq:momy} shows that in the dual space,
proximal ascent and gradient ascent are equivalent.  The actual
derivation is somewhat lengthy and is not reproduced here (but
see~\cite{bertsekas89} pages 244--245).  To make the connection with
proximal gradient explicit we can write an iterative formula for every
other element of this sequence
\begin{align*}
  y^{t+2} &= \arg\min_\lambda \{ -q(\lambda) + \frac{1}{2\rho}||\lambda -
  (y^t + \rho\grad q(y^t))||_2^2 \} \\
  &= \prox_{-q,1/\rho}(y^t + \rho\grad q(y^t)) \enspace.
\end{align*}

\section{Alternating direction method of multipliers}
The main difficulty we encounter with the method of multipliers is
that for the sparse coding problem, the joint minimization in
Equation~\ref{eq:momxz} is essentially as hard as the original
problem.  We can separate~\ref{eq:momxz} into separate minimizations
over $x$ and $z$ by doing a Gauss-Seidel pass over the two blocks
instead of carrying out the minimization directly.  This modification
leads to the following iteration:
\begin{align}
  x^{t+1} &= \arg\min_x L_\rho(x,z^{t},y^t) \label{eq:admmx} \enspace, \\
  z^{t+1} &= \arg\min_z L_\rho(x^{t+1},z,y^t) \label{eq:admmz} \enspace, \\
  y^{t+1} &= y^t + \rho (x^{t+1} - z^{t+1}) \nonumber \enspace,
\end{align}
where we lose the interpretation as proximal gradient on the dual function since
it is no longer true that $\grad q(y^t) \neq x^{t+1} - z^{t+1}$.  However, since
this method is tractable for our problems we focus on it over the method of
multipliers in the following chapters.


%% file: sparse-coding/sparse-coding.tex
\chapter{Sparse coding}
\label{chap:sparse-coding}

\section{Setting}
Sparse coding~\cite{olshausen96} is a feature learning and encoding
process, similar in many ways to the neural network based methods
discussed in the Introduction.  In sparse coding we construct a
dictionary $\{w_k\}_{k=1}^K$ which allows us to create accurate
reconstructions of input vectors $x$ from some data set.  It can be
very useful to consider dictionaries that are \emph{overcomplete},
where there are more dictionary elements than dimensions; however, in
this case minimizing reconstruction error alone does not provide a
unique encoding.  In sparse coding uniqueness is recovered by asking
the feature representation for each input to be as sparse as possible.

As with the neural network methods from the Introduction, there are
two phases to sparse coding:
\begin{enumerate}
\item A \emph{learning} phase, where the dictionary $\{w_k\}_{k=1}^K$
  is constructed, and
\item an \emph{encoding} phase, where we seek a representation of a
  new vector $x$ in terms of elements of the dictionary.
\end{enumerate}
In this report we focus on the encoding phase, and assume that the dictionary
$\{w_k\}_{k=1}^K$ is provided to us from some external source.  This focus of
attention is reasonable, since it was shown experimentally
in~\cite{coates2011encoding} that as long as the dictionary is constructed in a
reasonable way\footnote{ What exactly ``reasonable'' means in this context is an
  interesting question, but is beyond the scope of this report.  The results
  of~\cite{coates2011encoding} demonstrate that a wide variety of dictionary
  construction methods lead to similar classification performance, but do not
  offer conditions on the dictionary which guarantee good performance.}, then it
is the encoding process that has the most effect on classification performance.
When we want to make it explicit that we are considering only the encoding phase
we refer to the sparse encoding problem.

Formally, the sparse encoding problem can be written as an
optimization.  If we collect the dictionary elements into a matrix $W
= \begin{bmatrix} w_1 \,|\, \cdots \,|\, w_K \end{bmatrix}$ and denote
the encoded vector by $\hat{z}$ we can write the encoding problem as
\begin{align}
  \hat{z} = \arg\min_z \frac{1}{2}||Wz -x||_2^2 + \lambda||z||_1
  \enspace,
  \label{eq:sc}
\end{align}
where $\lambda \ge0$ is a regularization parameter that represents our
willingness to trade reconstruction error for sparsity.  This problem fits in
the framework of Chapter~\ref{chap:optimization} with $g(z) = 1/2||Wz - x||_2^2$
and $h(z) = \lambda||z||_1$.

Often it is useful to consider a non-negative version of sparse coding which
leads to the same optimization as Equation~\ref{eq:sc} with the additional
constraint that the elements of $z$ must be non-negative.  There are a few ways
we can formulate this constraint but the one that will be most useful to us in
the following chapters is to add an indicator function on the positive orthant
to the objective in Equation~\ref{eq:sc},
\begin{align}
  \hat{z} = \arg\min_z \frac{1}{2}||Wz -x||_2^2 + \lambda||z||_1 + \Pi(z)
  \enspace,
  \label{eq:nnsc}
\end{align}
where
\begin{align*}
    \Pi(z) = \begin{cases}
      0 &\text{if }z_i\ge0\quad \forall i \\
      \infty &\text{otherwise} \enspace.
    \end{cases}
\end{align*}

In most studies of sparse coding both the learning and encoding phases of the
problem are considered together.  In these cases, one proceeds by alternately
optimizing over $z$ and $W$ until convergence.  For a fixed $z$, the
optimization over $W$ in Equation~\ref{eq:sc} is quadratic and easily solved.
The optimization over $z$ is the same as we have presented here, but the matrix
$W$ changes in each successive optimization.  Since in our setting the
dictionary is fixed, we need only consider the optimization over $z$.

This difference in focus leads to a terminological conflict with the literature.
Since sparse coding often refers to both the learning and the encoding problem
together, the term ``non-negative sparse coding'' typically refers to a slightly
different problem than Equation~\ref{eq:nnsc}.  In Equation~\ref{eq:nnsc} we
have constrained only $z$ to be non-negative, whereas in the literature
non-negative sparse coding typically implies that both $z$ and $W$ are
constrained to be non-negative, as is done in~\cite{hoyer02}.  We cannot
introduce such a constraint here, since we treat $W$ as a fixed parameter
generated by an external process.

In the remainder of this chapter we introduce four algorithms for
solving the sparse encoding problem.  The first three are instances of
the proximal gradient framework presented in
Chapter~\ref{chap:optimization}.  The fourth algorithm is based on a
very different approach to solving the sparse encoding problem that
works by tracking solutions as the regularization parameter varies.

\section{Fast iterative soft thresholding}
\label{sec:fista}

Iterative soft thresholding (ISTA)~\cite{beckteboulle09} is the name
given to the proximal gradient algorithm with a fixed step size when
applied to problems of the form of Equation~\ref{eq:unconstrained},
when the non-smooth part $h$ is proportional to $||x||_1$.  The name
iterative soft thresholding arises because the proximal operator of
the $\ell_1$ norm is given by the soft threshold function
(see~\cite{murphy2012} $\mathsection$ 13.4.3.1):
\begin{align*}
  \prox_{\lambda||\cdot||_1}(x) = \soft_\lambda(x) =
  \operatorname{sign}(x)\max\{0, |x|-\lambda\} \enspace.
\end{align*}
In the case of sparse encoding this leads to iterations of the form
\begin{align*}
  z^{t+1} = \soft_{\lambda/L}(z^{t} - \frac{1}{L}W^\T(Wz^t - x)) \enspace.
\end{align*}
The constant $L$ here is the Lipschitz constant referred to in the statement of
Equation~\ref{eq:raw} which, in the case of sparse coding, is the largest
eigenvalue of $W^\T W$.  The ``Fast'' variant of iterative soft thresholding (FISTA)
modifies the above iteration to include a specially chosen momentum term,
leading to the following iteration, starting with $y^t = z^0$ and $k^1=1$:
\begin{align}
  z^{t} &= \soft_{\lambda/L}(y^t) \label{eq:fistaz}\enspace, \\
  k^{t+1} &= \frac{1 + \sqrt{1 + 4(k^t)^2}}{2} \label{eq:fistak} \enspace, \\
  y^{t+1} &= z^t + (\frac{k^t - 1}{k^{t+1}})(z^t - z^{t-1}) \label{eq:fistay} \enspace.
\end{align}
The form of the updates in FISTA is not intuative, but can be shown to lead to a
faster convergence rate than regular ISTA~\cite{beckteboulle09}.

\section{Sparse reconstruction by separable approximation}
\label{sec:sparsa}

Sparse reconstruction by separable approximation (SpaRSA)~\cite{wright09} is an
optimization framework designed for handling problems of the form considered in
Chapter~\ref{chap:optimization}.  This framework actually subsumes the ISTA and
FISTA style algorithms discussed above, but we consider a specific instantiation
of this framework which sets the step size using a
Barzilai-Borwein~\cite{barzborwein88} scheme, making it different from the
methods described above.  The development in~\cite{wright09} discusses SpaRSA in
its full generality, but specializing it to the sparse coding problem we get the
following iteration:
\begin{align*}
  z^{t+1} &= \soft_{\lambda/\alpha^t}(x^t - \frac{1}{\alpha^t} W^\T(Wz^t - x))
  \enspace,
  \\
  s^{t+1} &= z^{t+1} - z^t \enspace,
  \\
  \alpha^{t+1} &= \frac{||Ws^{t+1}||_2^2}{||s^{t+1}||_2^2} \enspace.
\end{align*}

The SpaRSA family of algorithms shares an important feature with other
Barzilai-Borwein methods, namely that it does not gaurentee a reduction in the
objective value at each step.  In fact, it has been observed that forcing these
methods to descend at each step (for example, by using a backtracking line
search) can significantly degrade performance in practice~\cite{wright09}.  In
order to guarantee convergence of this type of scheme a common approach is to
force the iterates to be no larger than the largest objective value in some
fixed time window.  This approach allows the objective value to occasionally
increase, while still ensuring that the iterates converge in the limit.

\section{Alternating direction method of multipliers}
\label{sec:admm}

The alternating direction method of multipliers (ADMM)~\cite{boyd2011} was
presented in Chapter~\ref{chap:optimization}.  Its instantiation for the sparse
encoding problem does not have its own name in the literature, but it can be
applied nonetheless.

For the sparse encoding problem the minimizations in Equations~\ref{eq:admmx}
and~\ref{eq:admmz} can be carried out in closed form.  This leads to the
updates:
\begin{align*}
  x^{k+1} &= (W^\T W + \rho I)^{-1}(W^\T x - \rho(z^k - \frac{1}{\rho}y^k))
  \enspace, \\
  z^{k+1} &= \soft_{\lambda/\rho}(x^{k+1} + \frac{1}{\rho}y^k) \enspace, \\
  y^{k+1} &= y^k + \rho(x^{k+1} - z^{k+1}) \enspace.
\end{align*}

\section{Boosted lasso}
\label{sec:blasso}

Boosted Lasso (BLasso)~\cite{zhao07} is a very different approach to
solving the sparse encoding problem than those considered above.
Rather than solving Equation~\ref{eq:sc} directly, BLasso works with
an alternative formulation of the sparse encoding problem,
\begin{align}
  \label{eq:blasso}
  \hat{z} = &\arg\min_z \frac{1}{2}||Wz - x||_2^2 \\
  \nonumber \text{st}&\quad ||z||_1 \le \beta \enspace.
\end{align}
For each value of $\lambda$ in Equation~\ref{eq:sc} there is a
corresponding value of $\beta$ which causes Equation~\ref{eq:blasso}
to have the same solution, although the mapping between values of
$\lambda$ and $\beta$ is problem dependent.  BLasso works by varying
the value of $\beta$ and maintaining a corresponding solution to
Equation~\ref{eq:blasso} at each step.

As the name suggests BLasso draws on the theory of Boosting, which can
be cast as a problem of functional gradient descent on the mixture
parameters of an additive model composed of weak learners.  In this
setting the weak learners are elements of the dictionary and their
mixing parameters are found in $z$.  Similarly to the algorithms
considered above, BLasso starts from the fully sparse solution but
instead of applying proximal iterations, it proceeds by taking two
types of steps: \emph{forward} steps, which decrease the quadratic
term in Equation~\ref{eq:blasso} and \emph{backward} steps which
decrease the regularizer.  In truth, BLasso also only gives exact
solutions to Equation~\ref{eq:blasso} in the limit as the step size
$\epsilon\to0$; however, setting $\epsilon$ small enough can force the
BLasso solutions to be arbitrarily close to exact solutions to
Equation~\ref{eq:blasso}.

BLasso can be used to optimize an arbitrary convex loss function with
a convex regularizer; however, in the case of sparse coding the
forward and backward steps are especially simple.  This simplicity
means that each iteration of BLasso is much cheaper than a single
iteration of the other methods we consider, although this advantage is
reduced by the fact that several iterations of BLasso are required to
produce reasonable encodings.  Another disadvantage of BLasso is that
it cannot be easily cast in a way that allows multiple encodings to be
computed simultaneously.


%% file: approximation/approximation.tex
\chapter{A reckless approximation}
\label{chap:approximation}

\section{Main result}

In this chapter we present the main result of this report, which is a connection
between the soft threshold features discussed in the Introduction, and the
sparse encoding problem.  Our key insight is to show how the soft threshold
features, (as defined in Equation~\ref{eq:st-features}) can be viewed as an
approximate solution to the non-negative sparse encoding problem
(Equation~\ref{eq:nnsc}).

We illustrate this connection through the framework of proximal gradient
minimization.  Using the tools presented in Chapter~\ref{chap:optimization} we
can demonstrate this connection by writing down a proximal gradient iteration
for the sparse encoding problem and computing the value of the first iterate,
starting from an appropriately chosen initial point.  We summarize this result
in a Proposition.

\begin{prop}
  The soft threshold features
  \begin{align*}
    z_k(x) = \max\{0, w_k^\T x - \lambda\}
  \end{align*}
  are given by a single step (of size 1) of proximal gradient descent
  on the non-negative sparse coding objective with regularization
  parameter $\lambda$ and known dictionary $W$, starting from the
  fully sparse solution.
\end{prop}
\begin{proof}
  Casting the non-negative sparse coding problem in the framework of
  Chapter~\ref{chap:optimization} we have
  \begin{align*}
    \min_z f(z) + g(z)
  \end{align*}
  with
  \begin{align*}
    g(z) &= \frac{1}{2}||Wz-x||_2^2 \enspace, \\
    h(z) &= \lambda||z||_1 + \Pi(z) \enspace.
  \end{align*}
  The proximal gradient iteration for this problem (with $\a=1$) is
  \begin{align*}
    z^{t+1} = \prox_{h}(z^t - W^\T(Wz^t - x)) \enspace.
  \end{align*}
  We now compute $\prox_h(x)$
  \begin{align*}
    u^* = \prox_h(x) &= \arg\min_u \Pi(u) + \lambda||u||_1 + \frac{1}{2}||u -
    x||_2^2 \enspace.
  \end{align*}
  This minimization is separable, and we can write down the solution for each
  element of the result independently:
  \begin{align}
    \label{eq:ustar}
    u_k^* = \arg\min_{u_k\ge 0} \lambda u_k + \frac{1}{2}(u_k - x_k)^2 \enspace.
  \end{align}
  Each minimization is quadratic in $u_k$, and therefore the optimum
  of Equation~\ref{eq:ustar} is given by $u_k^* = \max\{0, u_k^*\}$ where
  $u_k^*=x_k-\lambda$ is the unconstrained optimum.  We set $z^0=0$ and compute
  \begin{align*}
    z_k^1 &= \prox_h(z_k^0 - w_k^\T(Wz^0 - x)) \\
    &= \max\{0, w_k^\T x - \lambda\} \enspace,
  \end{align*}
  which is the desired result.
\end{proof}
Variants of the soft threshold features that appear in the literature
can be obtained by slight modifications of this argument.  For
example, the split encoding used in~\cite{coates2011encoding} can be
obtained by setting $W = \{w_k, -w_k\}$.

Once stated the proof of Proposition~1 is nearly immediate; however,
this immediacy only appears in
hindsight. In~\cite{coates2011encoding}, soft threshold features and
sparse coding features are treated as two separate and competing
entities (see, for example, Figure 1 in~\cite{coates2011encoding}).
In~\cite{NIPS2011_0668}, triangle features and sparse coding are
treated as two separate sparsity inducing objects.

From this Proposition we can draw two important insights:
\begin{enumerate}
\item Proposition~1 provides a nice explanation for the success of
  soft threshold features for classification.  Sparse coding is a well
  studied problem and it is widely known that the features from sparse
  coding models are effective for classification tasks.
\item On the other hand, Proposition~1 tells us that even very
  approximate solutions to the sparse coding problem are sufficient to
  build effective classifiers.
\end{enumerate}
Even optimizers specially designed for the sparse coding problem
typically take many iterations to converge to a solution with low
reconstruction error, yet here we see that a single iteration of
proximal gradient descent is sufficient to give features which have
been shown to be highly discriminative.

These insights open up three questions:
\begin{enumerate}
\item Is it possible to decrease classification error by doing a few
  more iterations of proximal descent?
\item Do different optimization methods for sparse encoding lead to different
  trade-offs between classification accuracy and computation time?
\item To what extent is high reconstruction accuracy a prerequisite to
  high classification accuracy using features obtained in this way?
\end{enumerate}

We investigate the answers to these questions experimentally in
Chapter~\ref{chap:experiments}, by examining how variants of the soft threshold
features preform.  We develop these variants by truncating other proximal
descent based optimization algorithms for the sparse coding problem.  The
remainder of this chapter presents ``one-step'' features from each of the
algorithms presented in Chapter~\ref{chap:sparse-coding}.

\section{Approximate FISTA}

Fast iterative soft thresholding was described in
Section~\ref{sec:fista}.  The first iteration of FISTA is a step of
ordinary proximal gradient descent since the FISTA iteration
(Equations~\ref{eq:fistaz},~\ref{eq:fistak} and~\ref{eq:fistay})
requires two iterates to adjust the step size.

Starting from $z^0=0$ the one step FISTA features are given by
\begin{align*}
  z^1 = \soft_{\lambda/L}(\frac{1}{L}W^\T x) = \frac{1}{L}\soft_{\lambda}(W^\T
  x) \enspace,
\end{align*}
which we see is equivalent to the soft threshold features, scaled by a
factor of $1/L$.

\section{Approximate SpaRSA}

Sparse reconstruction by separable approximation was described in
Section~\ref{sec:sparsa}.  SpaRSA, like FISTA, is an adaptive step
size selection scheme for proximal gradient descent which chooses its
step size based on the previous two iterates.  For the first iteration
this information is not available, and the authors of~\cite{wright09}
suggest a step size of 1 in this case.  This choice gives the
following formula for one step SpaRSA features:
\begin{align*}
  z^1 = \soft_{\lambda}(W^\T x) \enspace,
\end{align*}
which is exactly equivalent to the soft threshold features.

\section{Approximate ADMM}

The alternating direction method of multipliers was described in
Section~\ref{sec:admm}.  Since ADMM operates in the dual space,
computing its iterations requires choosing a starting value for the
dual variable $y^0$.  Since this choice is essentially arbitrary (the
value of $y^0$ does not effect the convergence of ADMM) we choose
$y^0=0$ for simplicity.  This leads to one step ADMM features of the
following form:
\begin{align}
  z^1 = \soft_{\lambda/\rho}((W^\T W + \rho I)^{-1}W^\T x) \enspace.
  \label{eq:oneadmm}
\end{align}
The parameter $\rho$ in the above expression is the penalty parameter
from ADMM.  As long as we are only interested in taking a single step
of this optimization, the matrix $(W^\T W + \rho I)^{-1}W^\T$ can be
precomputed and the encoding cost for ADMM is the same as for soft
threshold features.  Unfortunately, if we want to perform more
iterations of ADMM then we are forced to solve a new linear system in
$W^\T W + \rho I$ at each iteration, although we can cache an
appropriate factorization in order to avoid the full inversion at each
step.

The choice of $y^0=0$ allows us to make some interesting connections between the
one step ADMM features and some other optimization problems.  For instance, if
we consider a second order variant of proximal gradient (by adding a Newton term
to Equation~\ref{eq:proxgrad}) we get the following one step features for the
sparse encoding problem:
\begin{align*}
  z^1 = \soft_\lambda((W^\T W)^{-1}W^\T x)
  \enspace.
\end{align*}
We can thus interpret the one step ADMM features as a smoothed step of
a proximal version of Newton's method.  The smoothing in the ADMM
features is important because in typical sparse coding problems the
dictionary $W$ is overcomplete and thus $W^\T W$ is rank deficient.
Although it is possible to replace the inverse of $W^\T W$ in the
above expression with its pseudoinverse we found this to be
numerically unstable.  The ADMM iteration smooths this inverse with a
ridge term and recovers stability.

Taking a slightly different view of Equation~\ref{eq:oneadmm}, we can
interpret it as a single proximal Newton step on the Elastic
Net objective~\cite{zou05},
\begin{align*}
  \arg\min_z \frac{1}{2}||Wz - x||_2^2 + \rho||z||_2^2 +
  \frac{\lambda}{\rho}||z||_1 \enspace.
\end{align*}
Here the ADMM parameter $\rho$ trades off the magnitude of the
$\ell_2$ term, which smooths the inverse, and the $\ell_1$ term, which
encourages sparsity.

\section{Approximate BLasso}

BLasso was described briefly in Section~\ref{sec:blasso}.  Unlike the
other algorithms we consider, each iteration of BLasso updates exactly
one element of the feature vector $z$.  This means that taking one
step of BLasso leads to a feature vector with exactly one non-zero
element (of magnitude $\epsilon$).  In contrast, the proximal methods
described above update all of the elements of $z$ at each iteration,
meaning that the one step features can be arbitrarily dense.

This difference suggests that comparing one step features from the
other algorithms to one step features from BLasso may not be a fair
comparison.  To accommodate this, in the sequel we use many more steps
of BLasso than the other algorithms in our comparisons.

Writing out the one (or more) step features for BLasso is somewhat
notationally cumbersome so we omit it here, but the form of the
iterations can be found in~\cite{zhao07}.


%% file: experiments/experiments.tex
\chapter{Experiments}
\label{chap:experiments}

\section{Experiment 1}

We evaluate classification performance using features obtained by
approximately solving the sparse coding problem using the different
optimization methods discussed in Chapter~\ref{chap:sparse-coding}.
We report the results of classifying the
CIFAR-10\footnote{\url{http://www.cs.toronto.edu/~kriz/cifar.html}}
and STL-10\footnote{\url{http://www.stanford.edu/~acoates/stl10/}}
data sets using an experimental framework similar
to~\cite{coates2010analysis}.  The images in STL-10 are $96\times96$
pixels, but we scale them to $32\times32$ to match the size of
CIFAR-10 in all of our experiments.  For each different optimization
algorithm we produce features by running different numbers of
iterations and examine the effect on classification accuracy.

\subsection{Procedure}
\label{ssec:experiment-1-procedure}

\subsubsection{Training}

During the training phase we produce a candidate dictionary $W$ for
use in the sparse encoding problem.  We found the following procedure
to give the best performance:
\begin{enumerate}
\item Extract a large library of $6\times 6$ patches from the training
  set.  Normalize each patch by subtracting the mean and dividing by
  the standard deviation, and whiten the entire library using
  ZCA~\cite{bellsejnowski97}.
\item Run K-means with 1600 centroids on this library to produce a
  dictionary for sparse coding.
\end{enumerate}
We experimented with other methods for constructing dictionaries as
well, including using a dictionary built by omitting the K-means step
above and using whitened patches directly.  We also considered a
dictionary of normalized random noise, as well as a smoothed version
of random noise obtained by convolving noise features with a Guassian
filter.  However, we found that the dictionary created using whitened
patches and K-means together gave uniformly better performance, so we
report results only for this choice of dictionary.  An extensive
comparison of different dictionary choices appears
in~\cite{coates2011encoding}.

\subsubsection{Testing}

In the testing phase we build a representation for each image in the
CIFAR-10 and STL-10 data sets using the dictionary obtained during
training.  We build representations for each image using patches as
follows:
\begin{enumerate}
\item Extract $6\times 6$ patches densely from each image and whiten
  using the ZCA parameters found during training.
\item Encode each whitened patch using the dictionary found during
  training by running one or more iterations of each of the algorithms
  from Chapter~\ref{chap:sparse-coding}.
\item For each image, pool the encoded patches in a $2\times2$ grid,
  giving a representation with 6400 features for each image.
\item Train a linear classifier to predict the class label from these
  representations.
\end{enumerate}
This procedure involves approximately solving Equation~\ref{eq:sc} for
each patch of each image of each data set, requiring the solution to
just over $4\times10^7$ separate sparse encoding problems to encode
CIFAR-10 alone, and is repeated for each number of iterations for each
algorithm we consider.  Since iterations of the different algorithms
have different computational complexity, we compare classification
accuracy against the time required to produce the encoded features
rather than against number of iterations.

We performed the above procedure using features obtained with
non-negative sparse coding as well as with regular sparse coding, but
found that projecting the features into the positive orthant always
gives better performance, so all of our reported results use features
obtained in this way.

Parameter selection is performed separately for each algorithm and
data set.  For each algorithm we select both the algorithm specific
parameters, $\rho$ for ADMM and $\epsilon$ for BLasso, as well as
$\lambda$ in Equation~\ref{eq:sc}, in order to maximize classification
accuracy using features obtained from a single step of
optimization.\footnote{For BLasso we optimized classification after 10
  steps instead of 1 step, since 1 step of BLasso produced features
  with extremely poor performance.}  The parameter values we used in
this experiment are shown in Table~\ref{tab:param-values}.

\begin{table}[tbh]
  \centering
    \begin{tabular}{|l|c|c|c|}
      \hline
      \textbf{Algorithm} & \textbf{$\lambda$} & \textbf{$\rho$} & \textbf{$\epsilon$}\\ \hline
      BLasso & -- & -- & 0.25 \\
      FISTA & 0.1 & -- & -- \\
      ADMM & 0.02 & 30 & -- \\
      SPARSA & 0.1 & -- & -- \\
      \hline
    \end{tabular}
    \caption{Parameter values for each algorithm found by optimizing
      for one step classification accuracy (10 steps for BLasso).  Dashes indicate that the
      parameter is not relevant to the corresponding algorithm.}
  \label{tab:param-values}
\end{table}

\subsection{Results}

The results of this experiment on CIFAR-10 are summarized in
Figure~\ref{fig:perf-cifar} and the corresponding results on STL-10
are shown in Figure~\ref{fig:perf-stl}.  The results are similar for
both data sets; the discussion below applies to both CIFAR-10 and
STL-10.

The first notable feature of these results is that BLasso leads to
features which give relatively poor performance.  Although approximate
BLasso is able to find exact solutions to Equation~\ref{eq:sc},
running this algorithm for a limited number of iterations means that
the regularization is very strong.  We also see that for small numbers
of iterations the performance of features obtained with FISTA, ADMM
and SpaRSA are nearly identical.  Table~\ref{tab:accuracy} shows the
highest accuracy obtained with each algorithm on each data set over
all runs.

Another interesting feature of BLasso performance is that, when
optimized to produce one-step features, the other algorithms are
significantly faster than 10 iterations of BLasso.  This is
unexpected, because the iterations of BLasso have very low complexity
(much lower than the matrix-vector multiply required by the other
algorithms).

The reason the BLasso features take longer to compute comes from the
fact that it is not possible to vectorize BLasso iterations across
different problems.  To solve many sparse encoding problems
simultaneously one can replace the vectors $z$ and $x$ in
Equation~\ref{eq:sc} with matrices $Z$ and $X$ containing the
corresponding vectors for several problems aggregated into columns.
We can see from the form of the updates described in
Chapter~\ref{chap:sparse-coding}, that we can replace $z$ and $x$ with
$Z$ and $X$ without affecting the solution for each problem.  This
allows us to take advantage of optimized BLAS libraries to compute the
matrix-matrix multiplications required to solve these problems in
batch.  It is not possible to take advantage of this optimization with
BLasso.

\begin{figure}[tbh]
  \centering
  \includegraphics[width=1.0\linewidth]{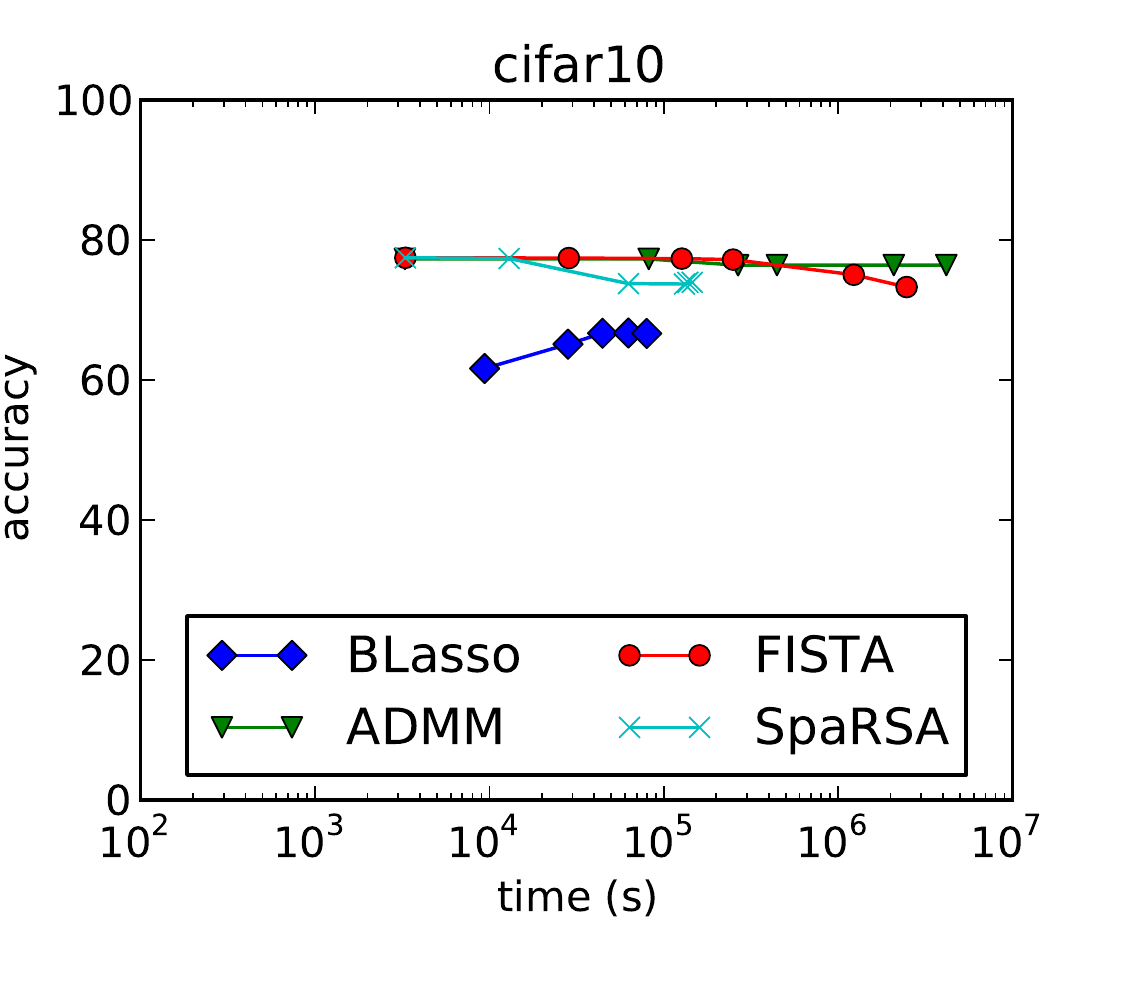}
  \caption{Classification accuracy versus computation time on CIFAR-10
    using different sparse encoding algorithms.  For FISTA, ADMM and
    SpaRSA the markers show performance measured with a budget of 1,
    5, 10, 50, 100 iterations.  The left-most marker on each of these
    lines shows performance using an implementation optimized to
    perform exactly one step of optimization.  The line for BLasso
    shows performance measured with a budget of 10, 50, 200, 500
    iterations.  In all cases early stopping is allowed if a
    termination criterion has been met (which causes some markers to
    overlap).}
  \label{fig:perf-cifar}
\end{figure}

\begin{figure}[tbh]
  \centering
  \includegraphics[width=1.0\linewidth]{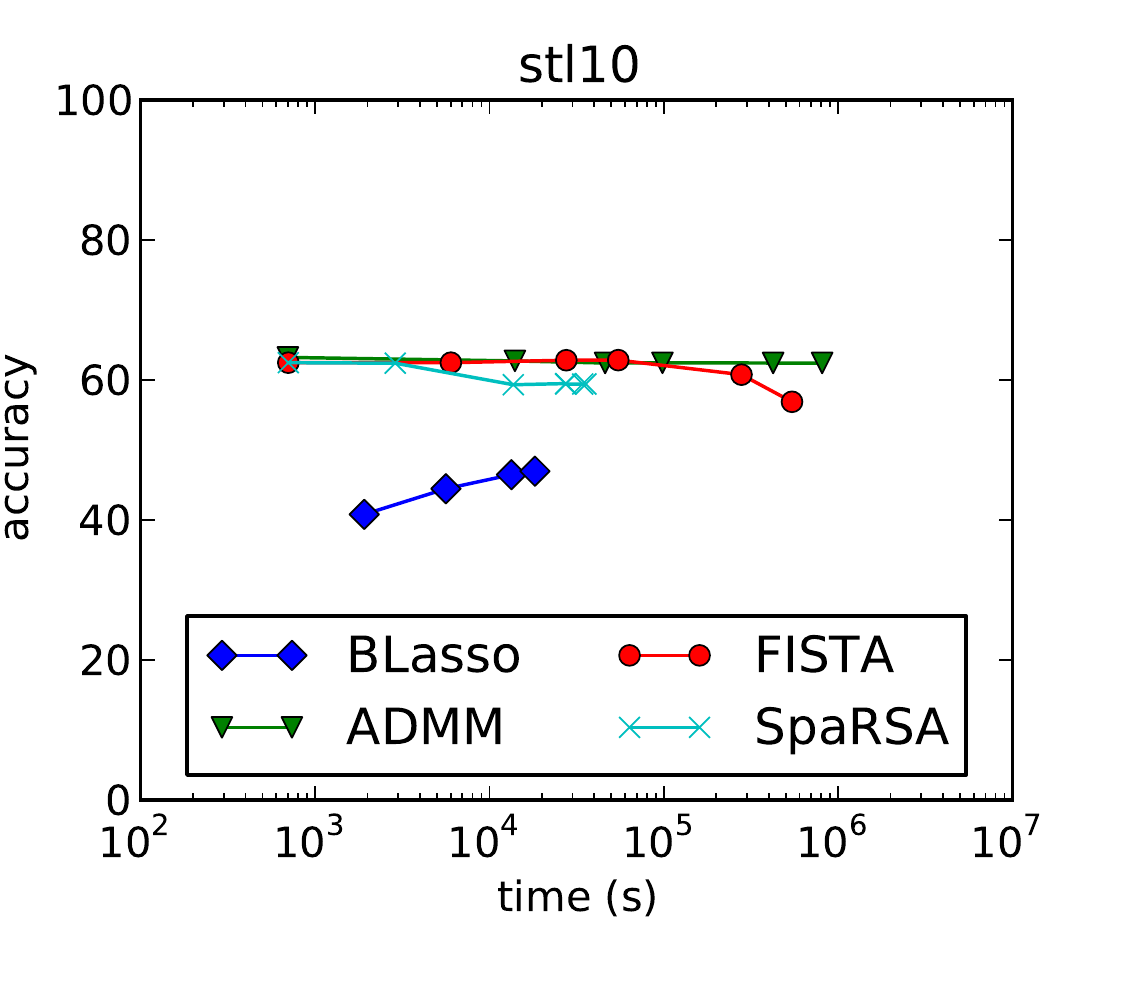}
  \caption{Classification accuracy versus computation time on STL-10
    using different sparse encoding algorithms.  For FISTA, ADMM and
    SpaRSA the markers show performance measured with a budget of 1,
    5, 10, 50, 100 iterations.  The left-most marker on each of these
    lines shows performance using an implementation optimized to
    perform exactly one step of optimization.  The line for BLasso
    shows performance measured with a budget of 10, 50, 200, 500
    iterations.  In all cases early stopping is allowed if a
    termination criterion has been met (which causes some markers to
    overlap).}
  \label{fig:perf-stl}
\end{figure}

The most notable feature of Figures~\ref{fig:perf-cifar}
and~\ref{fig:perf-stl} is that for large numbers of iterations the
performance of FISTA and SpaRSA actually drops below what we see with
a single iteration, which at first blush seems obviously wrong.  It is
important to interpret the implications of these plots carefully.  In
these figures all parameters were chosen to optimize classification
performance for one-step features.  There is no particular reason one
should expect the same parameters to also lead to optimal performance
after many iterations.  We have found that the parameters one obtains
when optimizing for one-step classification performance are generally
not the same as the parameters one gets by optimizing for performance
after the optimization has converged.  It should also be noted that
the parameter constellations we found in Table~\ref{tab:param-values}
universally have a very small $\lambda$ value.  This contributes to
the drop in accuracy we see with FISTA especially, since this
algorithm becomes unstable after many iterations with a small
$\lambda$.

This observation is consistent with the experiments
in~\cite{coates2011encoding} which found different optimal values for
the sparse coding regularizer and the threshold parameter in the soft
threshold features in their comparisons.  It should also be stated
that, as reported in~\cite{coates2011encoding}, the performance of
soft threshold features and sparse coding is often not significantly
different.  We refer the reader to the above cited work for a
discussion of the factors governing these differences.

\begin{table}[tbh]
  \centering

    \begin{tabular}{|l|c|}
      \hline
      \textbf{CIFAR10} & \textbf{Accuracy} \\ \hline
      BLasso & 66.7 \\ \hline
      FISTA & \textbf{77.5} \\  \hline
      ADMM & 77.3 \\ \hline
      SpaRSA & \textbf{77.5} \\ \hline
    \end{tabular}
    \qquad\quad
    \begin{tabular}{|l|c|}
      \hline
      \textbf{STL10} & \textbf{Accuracy} \\ \hline
      BLasso & 47.0 \\ \hline
      FISTA & 62.8 \\  \hline
      ADMM & \textbf{63.2} \\ \hline
      SpaRSA & 62.5 \\ \hline
    \end{tabular}

    \caption{Test set accuracy for the of the best set of parameters found in
      Experiment~1 on CIFAR-10 and STL-10 using features obtained by each different
      algorithm.}
  \label{tab:accuracy}
\end{table}

\section{Experiment 2}

This experiment is designed to answer our third question from
Chapter~\ref{chap:approximation}, relating to the relationship between
reconstruction and classification accuracy.  In this experiment we
measure the reconstruction accuracy of encodings obtained in the
previous experiment.

\subsection{Procedure}

\subsubsection{Training}

The encoding dictionary used for this experiment was constructed using
the method described in Section~\ref{ssec:experiment-1-procedure}.  In
order to ensure that our results here are comparable to the previous
experiment we actually use the same dictionary in both.

\subsubsection{Testing}

In order to measure reconstruction accuracy we use the following
procedure:
\begin{enumerate}
\item Extract a small library of $6\times6$ patches from randomly chosen
  images in the CIFAR-10 training set and whiten using the ZCA
  parameters found during training.
\item Encode each whitened patch using the dictionary found during
  training by running one or more iterations of each of the algorithms
  from Chapter~\ref{chap:sparse-coding}.
\item For each of the encoded patches, we measure the reconstruction
  error $||Wz - x||_2$ and report the mean.
\end{enumerate}

\subsection{Results}

The results of this experiment are shown in
Figure~\ref{fig:reconstruct}.  Comparing these results to the previous
experiment, we see that there is surprisingly little correlation
between the reconstruction error and classification performance.  In
Figure~\ref{fig:perf-cifar} we saw one-step features give the best
classification performance of all methods considered, here we see that
these features also lead to the worst reconstruction.

Another interesting feature of this experiment is that the parameters
we found to give the best features for ADMM actually lead to an
optimizer which makes no progress in reconstruction beyond the first
iteration.  To confirm that this is not merely an artifact of our
implementation we have also included the reconstruction error from
ADMM run with an alternative setting of $\rho=1$ which gives the
lowest reconstruction error of any of our tested methods, while
producing inferior performance in classification.

\begin{figure}[ht]
    \centering
    \includegraphics[width=1.0\textwidth]{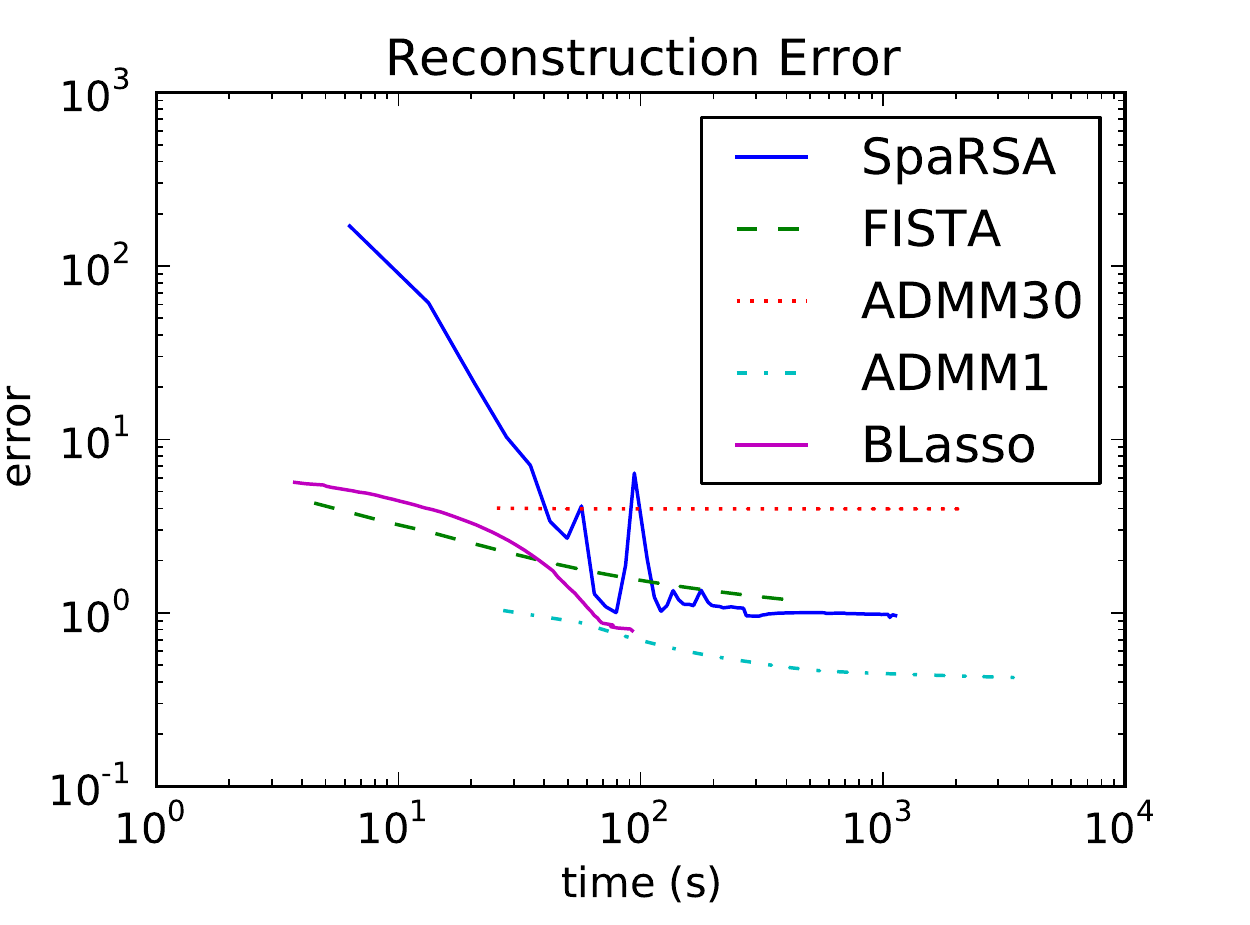}
    \caption{Mean reconstruction error versus computation time on a small sample
      of patches from CIFAR-10.  FISTA, ADMM and SpaRSA were run for 100
      iterations each, while BLasso was run for 500 iterations.  ADMM30
      corresponds to ADMM run with parameters which gave the best one-step
      classification performance.  ADMM1 was run with a different $\rho$
      parameter which leads to better reconstruction but worse classification.}
  \label{fig:reconstruct}
\end{figure}


%% file: conclusion/conclusion.tex
\chapter{Conclusion}
\label{chap:conclusion}

In this report we have shown that the soft threshold features, which
have enjoyed much success recently, arise as a single step of proximal
gradient descent on a non-negative sparse encoding objective.  This
result serves to situate this surprisingly successful feature encoding
method in a broader theoretical framework.

Using this connection we proposed four alternative feature encoding
methods based on approximate solutions to the sparse encoding problem.
Our experiments demonstrate that the approximate proximal-based
encoding methods all lead to feature representations with very similar
performance on two image classification benchmarks.

The sparse encoding objective is based around minimizing the error in
reconstructing an image patch using a linear combination of dictionary
elements.  Given the degree of approximation in our techniques, one
would not expect the features we find to lead to accurate
reconstructions.  Our second experiment demonstrates that this
intuition is correct.

An obvious extension of this work would be to preform a more thorough empirical
exploration of the interaction between the value of the regularization parameter
$\lambda$ and the degree of approximation.  From our experimentation we can see
only that such an interaction exists, but not glean insight into its structure.
Some concrete suggestions along this line are:
\begin{enumerate}
\item Preform a full parameter search for multi-step feature encodings.
\item Evaluate the variation of performance across different dictionaries.
\end{enumerate}
A full parameter search would make it possible to properly asses the usefulness
of performing more than one iteration of optimization when constructing a
feature encoding.  In this report we evaluate the performance of different
feature encoding methods using a single dictionary; examining the variability in
performance across multiple dictionaries would lend more credibility to the
results we reported in Chapter~\ref{chap:experiments}.

Another interesting direction for future work on this problem is an
investigation of the effects of different regularizers on approximate solutions
to the sparse encoding problem.  The addition of an indicator function to the
regularizer in Equation~\ref{eq:nnsc} appears essential to good performance and
proximal methods, which which we found to be very effective in this setting,
work by adding a quadratic smoothing term to the objective function at each
step.  The connection between one-step ADMM and the Elastic Net is also notable
in this regard.  Understanding the effects of different regularizers
empirically, or better yet having a theoretical framework for reasoning about
the effects of different regularizers in this setting, would be quite valuable.

In this work we have looked only at unstructured variants of sparse
coding.  It may be possible to extend the ideas presented here to the
structured case, where the regularizer includes structure inducing
terms~\cite{gregor2011}.  Some potential launching points for this
are~\cite{jenatton2010} and~\cite{jenatton2011}, where the authors
investigate proximal optimization methods for structured variants of
the sparse coding problem.
